\documentclass{6_15_17_Discovery}
\begin{document}

\ShortArticleName{Algebraic Formulation of a Data Set}
\AuthorNameForHeading{Wenqing (William) Xu}
\ArticleName{Deriving Compact Laws Based on Algebraic Formulation of a Data Set}
\Author{Wenqing (William) Xu$^\dagger$ and Mark Stalzer$^*$}

\Address{$^\dagger$California Institute of Technology
\EmailD{wxu@caltech.edu} 
}
\Address{$^*$California Institute of Technology
\EmailD{stalzer@caltech.edu} 
}
\begin{abstract}
In various subjects, there exist compact and consistent relationships between input and output parameters. Discovering the relationships, or namely compact laws, in a data set is of great interest in many fields, such as physics, chemistry, and finance. While data discovery has made great progress in practice thanks to the success of machine learning in recent years, the development of analytical approaches in finding the theory behind the data is relatively slow. In this paper, we develop an innovative approach in discovering compact laws from a data set. By proposing a novel algebraic equation formulation, we convert the problem of deriving meaning from data into formulating a linear algebra model and searching for relationships that fit the data. Rigorous proof is presented in validating the approach. The algebraic formulation allows the search of equation candidates in an explicit mathematical manner. Searching algorithms are also proposed for finding the governing equations with improved efficiency. For a certain type of compact theory, our approach assures convergence and the discovery is computationally efficient and mathematically precise.

\end{abstract}
\newpage
\section{Introduction}
Data driven discovery, which involves finding meaning and patterns in data, has been experiencing significant progress in quantifying behaviors, complexity, and relationships among data sets \cite{stalzer}. In various subjects, such as physics, chemistry, and finance, there exist relationships between various parameters. These relationships can be discovered by proofs, conjecture, or approximated using assumptions via the scientific method.  The scientific method has been the mainstay of understanding the laws that govern the universe. The method is based on observations that provide data in order to develop theories to predict future observations. It often starts with a mathematical hypothesis which is then verified by seeing how well the observed data fits a hypothesized model. However, the scientific method is rarely applied in the converse, formulating a plausible mathematical hypothesis using data.
Algorithms are developed to autonomously discover these relationships using sets of data alone, organized into input and output data. This is accomplished by proposing a series of candidate equations, plugging in the values of the data set into each equation, and determining how well the data fits each equation, typically using least squared methods. Notable progress has been made in applying different approaches \cite{predictive}.
$\\$
$\\$
Despite the progresses, challenges still exist. One problem with existing approaches is that without any assumptions on the relationship's format, arriving at the desired candidate equation is computationally slow. To address this, there are various algorithms that enumerate through these candidates. These algorithms must use some method of quantifying equation complexity in order to organize its enumeration and ensure every candidate equation of that complexity is written. One such method is the representation of an equation as a tree, where the nodes represent operators, the leaves represent data, and the complexity calculated as the sum of nodes and leaves \cite{Distilling}. However, this still means that the number of candidates increases exponentially with respect to complexity, meaning any brute-force algorithm causes high complexity equations to take an unreasonable amount of time to reach and verify against the data \cite{mining}.
$\\$
$\\$
The second problem is about constants. Many natural relationships have constants as part of their equations. An algorithm that enumerates through candidate equations does not take into account its constants. The Pareto frontier technique can be used to calculate these constants for the candidate equations. However, this method only gives you an approximation of the constants. In addition, this method does not explicitly rule out any candidate equations, as it accepts candidate equations and constants with a squared residual within a bound \cite{Distilling}. The sparsity of a given data set can also be used to bound the coefficients of the desired compact law \cite{sparse}. To the knowledge of the authors, there is currently no algorithm that explicitly rejects candidate equations that cannot be fitted with constants to the data, and also explicitly finds constants that allow candidate equations to be fitted to the data.
$\\$
$\\$
The third problem is on narrowing down the enumeration for the candidate equations. Brute-force methods lead to an unacceptable program run time \cite{runtime}. To combat this, algorithms have employed genetic algorithms and neural networks to introduce speed-ups in the program. Genetic algorithms introduce slight mutations in a candidate equation to single out operators and constants that fit the data well \cite{geneticdata}. Mutated equations that are promising, through some metric, generate equations with similar attributes, some with further mutations. This process is repeated until a candidate equation is found that can fit the data \cite{genetic}.
$\\$
$\\$
To combat the challenges, various approaches have been proposed. Machine learning based on the Neural Networks has shown its effective way in developing relationships using high throughput experimental data in novel ways \cite{machinelearning}. It is recognized that, for many applications, it is far easier to train a system using desired input-output examples than enumerating rules to obtain the desired response. Although convergence to the desired input-output relationships can be achieved via intensive and extensive training, there are no methods to prove that these machine learning algorithms converge onto a natural law based on the data \cite{practical}.
$\\$
$\\$
A new and prospective area of data-driven discovery is the development of automated science. Automated science involves creating algorithms that analyze data sets in order to create compact laws governing that data. A compact law refers to mathematically explicit description or equation that exactly describes the data \cite{compactlaw}. Among recent advances, one approach is based on statistical and model driven methods, for example, the use of Bayesian probabilistic methods and Markov models \cite{probability} as the basis of an intelligent system \cite{Prob}, and the expectation maximization algorithm, which converges to a maximum likelihood estimate based on incomplete data \cite{expectation}. Linear algebraic methods are applied to this field in order to improve run time and ensure convergence to a compact law. One such method is the use of randomized algorithms to decompose and diagonalize sparse matrices. As a result, this can be used to approximate a time dependent system, such as the Maxwell equations, using a Markov model, and thus approximate a system's behavior over time \cite{parallel}. In addition, the use of proper orthogonal decomposition on some sets of data can identify linearly dependent data sets to quickly classify bifurcation regimes in non-linear dynamical systems \cite{compressive}. Further, a method of using a library of functions acting on a sparse vector of constants. The resulting linear equation is then evaluated according to the data. This method was used to re-derive the equations governing the chaotic Lorentz system and fluid vortex shedding behind an obstacle \cite{discovering}. However, this method relies on separating data into independent and dependent variables. The independent variables are used to construct the function library and generate the compact law governing the dependent variables. Not all data can be cleanly separated into dependent and independent variables, so a method of incorporating all variables into a compact law is needed.
$\\$
$\\$
In this paper, we develop an innovative approach in discovering compact laws. We propose a novel algebraic equation formulation such that constant determination and candidate equation verification can be explicitly solved with low computational time. The algebraic formulation allows us to represent the evolution of equation candidates in an explicit mathematical manner. We also derive general methods for searching through a family of candidate equations and verifying them with respect to the data. The searching algorithms are defined conventionally and using a finite field to improve running time. The assumption of our approach on the compact law is that it is in some general format and incorporates only constants and variables related to the provided data. There are no assumptions on the data. We show that there is guaranteed convergence toward a valid equation candidate. Thus, for a specific type of compact theory, the discovery can be computationally efficient and mathematically precise. The proposed approach may have implications in many fields of data science, such as re-deriving natural laws of Physics, speculating in finance, and modeling chaotic, non-linear systems.
$\\$
$\\$
The paper is organized as follows. In section \ref{verif}, algebraic formulation is proposed for discovering equation candidates in a data set. An algorithm format to determine and verify if a candidate equation fits the data with respect to constants is presented. Proof of theorems for equation validation is shown in section \ref{validation}. In section \ref{algo}, search algorithms are proposed in finding candidate equations. We show the algorithm based on finite field sieve has improved performance over the exhaustive search algorithm. In section \ref{compmeth}, numerical results are presented using the proposed approach to re-derive the van der Waals equation from raw data, followed by concluding remarks in section \ref{conclusion}.

\section{Algebraic Formulation}
\label{verif}
\subsection{Virtual Experiment Setup}
\label{vexp}
Suppose we are presented with some data set, which are the inputs and outputs of some number of experiments. However, we do not know what data are the inputs and outputs. For each category of the data, every data entry must correspond to an experiment.
We can organize the data set in the format of a virtual experiment, as in a set of categories of data that relate to one another.
$\\$
Denote a set of $n$ characters representing the categories of our inputs/outputs for the virtual experiment as
\begin{equation}
F= \{F_1, \dots, F_n \}
\end{equation}
as the set of $n$ characters representing the categories of our inputs/outputs.
$\\$
For example, consider a virtual experiment that allows us to re-derive the classical force laws. This experiment involves two particles on a plane, some distance apart. Particle one has mass, charge, and velocity, while the other is a fixed mass and charge. Both are in a uniform magnetic field perpendicular to the plane. The set of input categories would be the masses of the two particles, $m_1,m_2$, the charges of the two particles, $q_1,q_2$, the distance between the two particles, $r$, the velocity of the first particle, $v_1$, and the strength of the magnetic field, $B$. This is denoted as
\[
F=\{ m_1, m_2, q_1, q_2, r, v_1, B, a \}.
\]
$\\$
As a result, the data sets which we will fit our force law equations onto will be of the form
\[
\{ m_{1,t}, m_{2,t}, q_{1,t}, q_{2,t}, r_t, v_{1,t}, B_{t}, a_t \}, 1 \leq t \leq r.
\]
The data from the force law experiment may be separated into input and output categories. However, given a data set, we need not assume whether each data category is an input or an output to discover an equation that describes the data. We can generalize the equations to only be in terms of the data set categories and constants, regardless of whether those categories are inputs or outputs.
\subsection{Equation Search Algorithm Format}
We define an equation search algorithm as a search algorithm that enumerates through equation candidates of a certain format, verifies them, and returns those that describe the data. The format for algebraic equation candidates, outputs of our search algorithm, will be of the form
\begin{equation}
0 = A_1 + \dots + A_s
\end{equation}
where all $A_i$, $1 \leq i \leq s$ is in the form
\begin{equation}
A_i=F_1^{f_1}\dots F_n^{f_n}, f_j \in Z^{+} \cup \{0\}, 1 \leq j \leq n,
\end{equation}
and for $\forall A_i$, $A_i \neq A_1, \dots, A_{i-1}, A_{i+1}, \dots$ under permutation of the elements of $F$.
$\\$
This format includes all possible algebraic equations involving elements of $F$ under constants.
$\\$
Let expression $A_i$ evaluated with values $F_{1,t}, \dots, F_{n,t}$ be denoted as $A_{i,t}$. We then define how an equation candidate is determined to describe the data set.
\begin{definition}
\label{valideq}
Define a ``valid equation candidate" of size $s$ and degree $d$ as an equation such that for all $r$ experiments, there exists a unique $k_2, \dots, k_s \in R$ such that for each experiment $t$,
\begin{equation}
0 = A_{1,t} + k_2A_{2,t} + \dots + k_sA_{s,t},
\end{equation}
and the maximum exponent of any $F_j$ of all $A_l$ is $d$.
\end{definition}
A valid equation candidate for the force laws \cite{classical} described in section $\ref{vexp}$ is $0=m_1m_2 + q_1q_2 + q_1v_1Br^2 + m_1ar^2$.
\subsection{Determination of Constants}
Suppose we have some equation candidate
\[
0 = A_1 + \dots + A_s.
\]
We then evaluate this equation for the $r$ virtual experiments, obtaining the numerical values of all $F_j$, and thus $A_i$. As a result, we can evaluate for the constants $k_2, \dots, k_s$ by solving the resulting matrix equation
\begin{equation}\label{datamat} \left[ \begin{array}{c}
-A_{1,1} \\
\vdots \\
-A_{1,r} \end{array}\right] =
\left[ \begin{array}{ccc}
A_{2,1} & \dots & A_{s,1}\\
\vdots & \ddots & \vdots \\
A_{2,r} & \dots & A_{s,r}\end{array}\right]
\left[ \begin{array}{c}
k_2 \\
\vdots \\
k_s  \end{array}\right].
\end{equation}
Equation (\ref{datamat}) is equivalent to the matrix equation
\begin{equation} \left[ \begin{array}{c}
0 \\
\vdots \\
0 \end{array}\right] =
\left[ \begin{array}{ccc}
A_{2,1} & \dots & A_{s,1}\\
\vdots & \ddots & \vdots \\
A_{2,r} & \dots & A_{s,r}\end{array}\right]
\left[ \begin{array}{c}
k_2 \\
\vdots \\
k_s  \end{array}\right]
- \left[ \begin{array}{c}
-A_{1,1} \\
\vdots \\
-A_{1,r} \end{array}\right].
\end{equation}
\begin{definition}
Let a data matrix of equation candidate
\[
0 = A_1 + \dots + A_s.
\]
with $r$ experiments be defined as
\begin{equation}
\left[ \begin{array}{ccc}
A_{2,1} & \dots & A_{s,1}\\
\vdots & \ddots & \vdots \\
A_{2,r} & \dots & A_{s,r}\end{array}\right]
\end{equation}
\end{definition}
\begin{definition}
Let $A^*$ denote the conjugate transpose of $A$ such that if
\[
A=\left[ \begin{array}{ccc}
A_{2,1} & \dots & A_{s,1}\\
\vdots & \ddots & \vdots \\
A_{2,r} & \dots & A_{s,r}\end{array}\right],
\]
\begin{equation}
A^*=\left[ \begin{array}{ccc}
\overline{A_{2,1}} & \dots & \overline{A_{2,r}}\\
\vdots & \ddots & \vdots \\
\overline{A_{s,1}} & \dots & \overline{A_{s,r}}\end{array}\right].
\end{equation}
where $\overline{A_{i,j}}$ is defined as the complex conjugate of $A_{i,j}$.
\end{definition}
\begin{definition}
\label{MPI}
The Moore-Penrose left pseudoinverse of $A \in M(m,n,R)$ is defined as $A^+ \in M(n,m,R),m,n \in Z^+$ such that
\begin{equation}
\label{mprdef}
A^+A=I,
\end{equation} the $n \times n$ identity matrix \cite{penrose}.
$\\$
If the columns of $A$ are linearly independent, then the Moore-Penrose left pseudoinverse is calculated as
\begin{equation}
\label{mprcalc}
A^+=(A^*A)^{-1}A^*
\end{equation}
such that $A^+A=((A^*A)^{-1}A^*)A=(A^*A)^{-1}(A^*A)=I$ (\cite{penrose}, Theorem 2).
\end{definition}
\noindent

\section{Equation Validation}
\label{validation}
\subsection{Theorems}
We show that the equation validation question (Definition \ref{valideq}) is equivalent to a linear algebra question. 
\begin{theorem}
\label{leftinverse}
If $s \geq 2$, $s \in Z^+$,
$\\$
The Moore-Penrose left pseudoinverse $A^+$ of data matrix
\[
A=\left[ \begin{array}{ccc}
A_{2,1} & \dots & A_{s,1}\\
\vdots & \ddots & \vdots \\
A_{2,r} & \dots & A_{s,r}\end{array}\right]
\]
can be computed and $(AA^+-I)\vec{b}=\vec{0}$ if and only if its corresponding equation candidate
\[
0 = A_1 + \dots + A_s
\]
is valid.
\end{theorem}
\begin{proof}
If the equation candidate
\[
0 = A_1 + \dots + A_s
\]
is valid, then by definition \ref{valideq}, the definition of the candidate equation being valid, there exists a unique vector
\[
\vec{k} = \left[ \begin{array}{c}
k_2 \\
\vdots \\
k_s  \end{array}\right], \text{ }k_2, \dots, k_s \in R \text{ such that}
\]
\begin{equation}
\left[ \begin{array}{c}
0 \\
\vdots \\
0 \end{array}\right] =
\left[ \begin{array}{ccc}
A_{2,1} & \dots & A_{s,1}\\
\vdots & \ddots & \vdots \\
A_{2,r} & \dots & A_{s,r}\end{array}\right]
\left[ \begin{array}{c}
k_2 \\
\vdots \\
k_s  \end{array}\right]
- \left[ \begin{array}{c}
-A_{1,1} \\
\vdots \\
-A_{1,r} \end{array}\right].
\end{equation}
$\\$
Thus, if we define
\begin{equation}
\vec{b}=
\left[ \begin{array}{c}
-A_{1,1} \\
\vdots \\
-A_{1,r} \end{array}\right],
\end{equation}
we obtain the relationship
\begin{equation}
\label{relat1}
A\vec{k} = \vec{b}.
\end{equation}
As a result, the least squares problem min$|A\vec{k} - \vec{b}|$ has a unique solution, and so the columns of $A$ are linearly independent (\cite{LSQ}, 2.4).
$\\$
From (definition \ref{MPI}), we see that if the columns of $A$ are linearly independent, $A^*A$ is invertible, and so the Moore-Penrose left pseudoinverse can be computed as $A^+=(A^*A)^{-1}A^*$ (\ref{mprcalc}). As a result,
multiplying both sides of (\ref{relat1}) by $A^+$ gives us
\[A^+A\vec{k}=A^+\vec{b}, \] and using (\ref{mprdef}) yields
\begin{equation}
\label{eigen}
\vec{k}=A^+\vec{b}.
\end{equation}
Substituting (\ref{eigen}) in (\ref{relat1}) and subtracting $\vec{b}$ from both sides obtains
\begin{equation}
\label{eigenresult}
(AA^+-I)\vec{b}=\vec{0}.
\end{equation}
Assume the Moore-Penrose left pseudoinverse of data matrix $A$ can be computed as $A^+$ and $(AA^+-I)\vec{b} = \vec{0}$. We then have equation
\begin{equation}
\label{init2}
(AA^+-I)\vec{b} = AA^+\vec{b}-\vec{b} = \vec{0}.
\end{equation}
Substituting $A^+\vec{b} = \vec{k}$ in (\ref{init2}) yields the desired result
\[
A\vec{k} - \vec{b} = \vec{0}.
\]
We also see that we obtain some $\vec{k} \in R^{s-1}$.

$\\$
As a result, there exists a unique $k_2, \dots, k_s$ such that for all experiments $t$, $1 \leq t \leq r$,
\[
0 = A_{1,t} + k_2A_{2,t} + \dots + k_sA_{s,t},
\]
so by definition \ref{valideq}, the equation candidate
\[
0 = A_1 + \dots + A_s
\]
is valid.
\end{proof}

\noindent We then show some theorems and corollaries that follow immediately from theorem \ref{leftinverse}. 
\begin{theorem}
For $F_i \in F,A_1$, $F_i=0$ is valid if and only if $0=A_1$ is valid.
\end{theorem}

\begin{proof}
If there $\exists F_i \in F,A_1$ such that $F_i = 0$ is valid, $F_{i,1}, \dots, F_{i,r} = 0$. Evaluating yields $A_{1,1}, \dots, A_{1,r} = 0$, and so $A_1=0$ is valid.
$\\$
If $A_1=0$ is valid, then assume string $A_1=B_1A_1^{'}$, where $B_1 \in F$. Since $R$ is a field, and all fields are integral domains \cite{abstract}, either $B_{1,1}, \dots, B_{1,r}=0$ or $A_{1,1}^{'}, \dots, A_{1,r}^{'} = 0$. If $B_1 \in F$, then $B_1=0$ is valid. Else, $A_1^{'}=0$ is valid, and so we repeat the previous procedure until we obtain $\exists B_i \in F,A_1$ such that $B_i=0$ is valid, or $A_1^{'}=1=0$ is valid, which is false in a field by definition.
\end{proof}

\begin{corollary}
\label{sr}
If $r < s-1$, $0 = A_1 + \dots + A_s$ cannot be valid.
\end{corollary}

\begin{proof}
If $r < s-1$, then there are more columns than rows in the candidate equation's data matrix
\[
A=\left[ \begin{array}{ccc}
A_{2,1} & \dots & A_{s,1}\\
\vdots & \ddots & \vdots \\
A_{2,r} & \dots & A_{s,r}\end{array}\right],
\]
and so the columns of $A$ are not linearly independent. Thus, $A^*A$ is not invertible (\cite{leftI}, Theorem 3), and so $A^+$ cannot be computed. By theorem $\ref{leftinverse}$, $0 = A_1 + \dots + A_s$ cannot be valid.
\end{proof}

\begin{corollary}
\label{square}
If data matrix $A$ is square and invertible, its corresponding candidate equation is valid.
\end{corollary}

\begin{proof}
If $A$ is square and invertible, we have $A^+A=I=AA^+$. As a result, $(AA^+-I)\vec{b}=[0]\vec{b}=\vec{0}$. By theorem $\ref{leftinverse}$, the corresponding candidate equation is valid.
\end{proof}

\begin{corollary}
The equation $(AA^+-I)\vec{b}=\vec{0}$ is valid if and only if $1.\vec{b}$ is the eigenvalue, eigenvector pair of matrix $AA^+$.
\end{corollary}
\begin{proof}
Assume $(AA^+-I)\vec{b}=\vec{0}$. As a result, since $\vec{b} \neq \vec{0}$ because of our equation candidate format, $\rVert AA^+-I \rVert=0$, and so $\vec{b}$ must be an eigenvector of $AA^+-I$.
$\\$
Assume $\vec{b}$ is an eigenvector of $AA^+-I$. Thus, by definition, $AA^+\vec{b}=\vec{b}$, and so $AA^+\vec{b}-\vec{b}=(AA^+-I)\vec{b}=\vec{0}$.
\end{proof}
%
%
%
%
Next we see that by re-expressing equation validation and constant determination as linear algebraic operations, the computational complexity of the validation question (definition \ref{valideq}) can be determined.
\subsection{Computational Complexity}
We describe this problem as validating an equation candidate in the family of equation candidates involving elements of $F$, $|F| = n$, of size at most $s$, and of degree at most $d$ that satisfies the data of $r$ experiments. We assume that all additive and multiplicative operations are floating point operations.
Evaluating each equation $A_{j,k}, 1 \leq j \leq s-1, 1 \leq k \leq r$, takes $nd$. Thus, constructing the matrix equation $A$ will take O($rsdn$) time.
$\\$
As a result, determining whether a given equation candidate is valid involves calculating $A^*A$, determining the existence of $(A^*A)^{-1}$, calculating $AA^+$, and the value of $(AA^+-I)\vec{b}$. These operations are done in O($rs^2$), O($s^3$), O($rs^2$), and O($rs$) time respectively, where $s$ is the size of the equation candidate and $r$ is the number of experiments. By (corollary $\ref{sr})$, we see that $r \geq s-1$. Thus, if
\begin{equation}
\label{tdef}
t=\max(dn,r),
\end{equation}
the running time of checking whether an equation candidate is valid is O($tr^2$). We can now develop a search algorithm that applies this algorithm repeatedly to many different candidate equations to find one that is valid.
\section{Equation Candidate Search}
\label{algo}
\subsection{Exhaustive Search Algorithm}
\label{exh}
We have demonstrated an algorithm that can verify whether a given candidate equation is valid. However, the second half of the problem of deriving algebraic equations that fit our data set is a search algorithm that finds valid equation candidates in a certain family of candidate equations. Referring to \cite{sentences}, we can denote this problem as finding a valid equation candidate in the family of equation candidates involving elements of $F$, $|F| = n$, of size at most $s$, and of degree at most $d$ that satisfies the data of $r$ experiments, where $r \geq s-1$ $($cor $\ref{sr})$. We also bound $d$ such that $d \leq {r / n}$.
$\\$

We describe an exhaustive search algorithm for this problem as follows.
$\\$
\begin{enumerate}
    \item Begin by finding all valid equation candidates of size 1. This simply entails finding all $F_i \in F$ such that $F_{i,t}=0$ for each experiment $t$.
    \item Then find all valid equation candidates of size 2. This is begun by enumerating through all
    $\\$
    $A_1=F_1^{f_1}\dots F_n^{f_n}, 0 \leq f_j \leq d, 1 \leq j \leq n$, and no $F_j$ is such that $F_{j,t}=0$ for each experiment $t$.
    $\\$
    For each $A_1$ we generate, we choose each $A_2$ generated as above such that $A_1 \neq A_2$. We then obtain a list of all possible equation candidates of the form $A_1 + A_2 = 0$.
    \item For each candidate $A_1 + A_2 = 0$, we apply (\ref{leftinverse}) to verify the equation candidate is valid.
    \item To find all valid equation candidates of size $k \leq s$, for each instance of $A_1 + \dots + A_{k-1} = 0$, we add an instance of $A_k$ generated as in step 2 such that $A_k \neq A_1, \dots,A_{k-1}$,
    \item Repeat the inductive step until you generate all candidate equations of size at most $s$.
\end{enumerate}
To generate each $A_i$, it requires at most $n^d$ steps. Since $dn \leq r$, from (\ref{tdef}), we see that $t = r$. Thus, to check the validity of all candidate equations of size at most $s$, the number of steps is
\begin{equation}
\label{exhcomp}
nr+\sum_{i=2}^s n^{id} tr^2 \approx n^{ds} tr^2 \leq n^{\frac{rs}{n}} r^3=\left(n^{\frac{1}{n}}\right)^{rs} r^3.
\end{equation}
Combining (\ref{exhcomp}) with the fact that $e^{\frac{1}{e}} \geq x^{\frac{1}{x}}$ for all $x \in R$ yields
\begin{equation}
\label{exhcomplex}
\left(n^{\frac{1}{n}}\right)^{rs} r^3 \leq \left(e^{\frac{1}{e}}\right)^{rs} r^3.
\end{equation}
Using (\ref{exhcomp}) and (\ref{exhcomplex}), we see that the time complexity of this algorithm is $e^{O(1)rs}$.
$\\$
One drawback of this exhaustive search algorithm is that, in order to enumerate through all equations, the exponents of each variable of each additive term must be enumerated through as well. \textbf{\textit{Using a property of finite fields, there is a method to find all valid equation candidates without parsing through all exponents of a variable.}}

\subsection{Finite Field Sieve Algorithm}
\subsubsection{Introduction to Finite Fields}
We will explain some relevant properties of finite fields $\cite{fforder}$. A finite field of order $p$ is some set $F_p$ such that $|F_p|=p$, and two operations $+,*$ that satisfy some properties: For all $s_1,s_2,s_3 \in F_p$,
\begin{enumerate}
\item $s_1*s_2,s_1+s_2 \in F_p$.
\item $(s_1+s_2)+s_3=s_1+(s_2+s_3). (s_1*s_2)*s_3=s_1*(s_2*s_3)$
\item $s_1+s_2=s_2+s_1$, $s_1*s_2=s_2*s_1$.
\item There exists a unique $0,1 \in F_p$ such that $0+s_1=s_1$,$1*s_1=s_1$.
\item There exists a unique $s_3,s_4$ such that $s_3+s_1 = 0$ and $s_4*s_1=1$.
\item $s_1*(s_2+s_3)=s_1*s_2+s_1*s_3$.
\end{enumerate}
One important property of a finite field is the order of an element $s \in F_p$. This is defined as the least exponent $d \in Z^+$ such that $s^d=1$. In a finite field where $p$ is prime, every non-zero element in $F_p$ order $1$ or a divisor of $p-1$.

\subsubsection{Algorithm Preliminaries}

Assume we have some valid equation candidate of some degree $d$. If the data in our data sets are of some finite field $F_p$, where $p$ is prime, then each $A_i=F_1^{f_1}\dots F_n^{f_n}$ in our candidate is congruent to $F_1^{f_1'}\dots F_n^{f_n'}$, where $f_1'=f_1$ mod p-1, $\dots$, $f_n'=f_n$ mod p-1. As a result, if a verification algorithm can be performed in a finite field, we can bound the exponents of the family of equation candidates to be searched through.
$\\$
We then show that, through a modified validation algorithm, it is possible to obtain a set of equation candidates such that one of those is a valid equation candidate.
\begin{theorem}
\label{ff}
Let there exists a homomorphism
\begin{equation}
\label{homo}
\varphi: Q \rightarrow Z/pZ,
\end{equation} where $p$ is prime, and let each $F_j \in Q,1 \leq j \leq n$. Also, let each $F_j'=\varphi(F_j), 1 \leq j \leq n$, and
\begin{equation}
A_i'=F_1'^{f_1}\dots F_n'^{f_n}, f_j \in Z^{+} \cup \{0\}, 1 \leq j \leq n.
\end{equation}
If $A_1 + \dots + A_s = 0$ is a valid equation candidate, then there exists a solution $\vec{x'} \in (Z/pZ)^{s-1}$ such that
\begin{equation}
\left[ \begin{array}{ccc}
A_{2,1}' & \dots & A_{s,1}'\\
\vdots & \ddots & \vdots \\
A_{2,r}' & \dots & A_{s,r}'\end{array}\right]
\vec{x'}
=\left[ \begin{array}{c}
-A_{1,1}' \\
\vdots \\
-A_{1,r}' \end{array}\right].
\end{equation}
\end{theorem}

\begin{proof}
Let $A$ be the equation matrix of $A_1 + \dots + A_s$ and
\[
\vec{b}=
\left[ \begin{array}{c}
-A_{1,1} \\
\vdots \\
-A_{1,r} \end{array}\right].
\]

We see that if there exists a unique $\vec{x} \in Q^{s-1}$ such that $A \vec{x}-\vec{b}=\vec{0}$, then $\exists k \in Z^+$, the least common denominator of all entries of $A$ and $\vec{b}$ such that $k \left( A \vec{x}-\vec{b} \right)=k\vec{0}=\vec{0}$, with $kA \in Z^{r,s-1},k\vec{b} \in Z^{r}$.
$\\$
We then see that $\exists \vec{x}'= k\vec{x}$ mod p such that $kA\vec{x}'-k\vec{b}=\vec{0}$ mod p.
$\\$
Define homomorphisms $\phi: Q \rightarrow kQ$, $\phi(q)=kq$, and $\Phi: Z \rightarrow Z/pZ$, $\Phi(kq)=kq$ mod p. Thus, we can define $\varphi: Q \rightarrow Z/pZ$ as $\varphi(q)=\Phi(\phi(q))$.
$\\$
Thus, we have that $\Phi(\phi(A \vec{x}-\vec{b}))=\varphi(A)\varphi(\vec{x})-\varphi(\vec{b})=\Phi(\phi(\vec{0}))=\vec{0}$, so there exists $\varphi(\vec{x})$ such that $\varphi(A)\varphi(\vec{x})=\varphi(\vec{b})$.
$\\$
\end{proof}
We can modify the algorithm in (\ref{leftinverse}) for solving a linear system $A\vec{x}=\vec{b}$ over a finite field in O$(tr^2)$ time if such a solution exists, where the dimensions of $A$ are $r \times s$ $\cite{ffsolve}$. As a result, there exists a polynomial time algorithm to "validate" the equation over the finite field.
$\\$
\subsubsection{Algorithm Description}
We will outline an algorithm to search for valid equation candidates of size at most $s$ and degree at most $d$ that satisfies the data of $r$ experiments, where $r \geq s-1$. Let $p \leq \sqrt{d}$ be some prime. In addition, bound $d$ such that $d \leq {r / n}$. Assume that all elements of our data set are floating point numbers and all operations are floating point operations.
$\\$
\begin{enumerate}
    \item Multiply each entry in the data set by a common denominator $10^k,k \in Z^+$.
    \item Take the modulo $3$ of all elements in the data set and place all values in a duplicate data set $D_3$.
    \item Perform the exhaustive search algorithm in section $\ref{exh}$ on equation candidates of size $s$, degree at most $3-1=2$, that satisfies the data in $D_3$ of the $r$ experiments. For the verification algorithm, use (\ref{leftinverse}), but modified to apply to finite fields $\cite{ffsolve}$, to write down the validated equation candidates in $F_3$ in some list $EqF_3$.
    \item Repeat steps 2 and 3 for a finite field of order 5, order 7, $\dots, p$.
    \item Take equation candidates of size at most $s$ and degree $d$ such that, upon taking modulo 3 of each exponent, yields an equation candidate in $EqF_3$, taking modulo 5 of each exponent, yields an equation candidate in $EqF_5$, $\dots$, taking modulo p of each exponent, yields an equation candidate in $EqF_p$, and write them down in $FFV$.
    \item Validate the equation candidates in $FFV$ using the original data set and the algorithm denoted in $\ref{leftinverse}$ to obtain the valid equation candidates.
\end{enumerate}
$\\$
To multiply entry in the data set by a common denominator $10^k,k \in Z^+$ takes $nr$ operations. Taking the modulo $p$ of each operation takes $nr$O$(1)$ operations. Since $dn \leq r$, from (\ref{tdef}), we see that $t = r$. We then see that $p \leq \sqrt{d} \leq \sqrt{{r / n}}$.
Thus, using (\ref{exhcomp}) and (\ref{exhcomplex}), checking the validity of all candidate equations of size at most $p$ is approximately
\begin{equation}
\label{ff1}
n^{\sqrt{\frac{r}{n}}}r^3)  \approx e^{O(1)\sqrt{r}s}.
\end{equation}
$\\$
To find all equation candidates in $FFV$ involves first finding equation candidates in $EqF_3$ equivalent to, with respect to modulo 3 of the exponents, equation candidates in $EqF_5$, with respect to modulo 5 of the exponents, and so on, and placing them in $FFV'$. This is accomplished by solving $n$ linear equations, at most $p$ times, on at most $n^{3s}$ equation candidates. This takes $n^{3s+1}O(d^2)p$. This shows that $|FFV'|=k$ is approximately constant.
$\\$
We also must take into account that the exponents of those equation candidates in $FFV'$ are in modulo $p$. For each exponent, there are $p$ possible exponents because $p \leq \sqrt{d}$. As a result, there are $k\left(n^{s}\right)^p$ equation candidates in $FFV$ that we must validate using the algorithm denoted in $\ref{leftinverse}$. Thus, using (\ref{exhcomp}) and (\ref{exhcomplex}), the running time is \begin{equation}
\label{ff2}
k\left(n^{s}\right)^p r^3 \approx e^{O(1)s\sqrt{r}} r^3.
\end{equation}
$\\$
Thus, adding the running times of (\ref{ff1}) and (\ref{ff2}) yields the running time of the Finite Field Sieve, which is
\begin{equation}
\label{ff_final}
e^{O(1)\sqrt{r}s}.
\end{equation}
\section{Numerical Results} 
\label{compmeth}

Equation candidates for the van der Waals equation of state \cite{vanderwaals},
\begin{equation}
\begin{split}
\left(P+a\left(\frac{n}{V}\right)^2 \right)\left(\frac{V}{n}-b\right) =
RT \implies
\\
PV^3 - bnPV^2+an^2V-abn^3-RnV^2T =
\\
PV^3+k_2nPV^2+k_3n^2V+k_4n^3+k_5nV^2T=0,
\end{split}
\end{equation}
\noindent were generated and verified using simulated data.
$\\$
\noindent The pressure $P$, volume $V$, and number of moles $n$ for 20 different virtual experiments were generated randomly. The van der Waals coefficients $a=2.45 \times 10^{-2},b=2.661\times 10^{-5}$ were values for hydrogen, the gas constant $R$ was set as 8.3145. For each experiment, the temperature $T$ was calculated.
$\\$
The set of inputs and outputs that compose our candidate equations is $L \cup F = \{P,V,n,T\}$. The data points are recorded in Table \ref{table_Waals}.
\begin{table}[h]
 \centering
 \caption{Recorded Data Points for Van Der Waals Test} \label{table_Waals}%
 \begin{tabular}{|c | c | c | c | c|}
 \hline
 Experiment Number & P & V & T & n \\ [0.5ex]
 \hline\hline
 1 & 3 & 2 & 2 & 0.186219 \\
 \hline
 2 & 2 & 3 & 4 & 0.362773\\
 \hline
 3 & 4 & 4 & 5 & 0.38854\\
 \hline
 4 & 5 & 1 & 8 & 0.0987221 \\
 \hline
 5 & 6 & 5 & 1 & 3.60872\\
  \hline
 6 & 7 & 11 & 2.5 & 3.70502\\
  \hline
 7 & 9 & 3 & 2.2 & 1.4782\\
  \hline
 8 & 2.5 & 4.1 & 7.3 & 0.174113\\
  \hline
 9 & 5.3 & 6.4 & 9.7 & 0.425028\\
  \hline
 10 & 4.4 & 3.2 & 8.2 & 0.214052\\[1ex]
 \hline
\end{tabular}
\end{table}
We searched in the family of algebraic equation candidates of additive size $5$ and exponent order $5$.
The exhaustive search and finite field sieve methods were used to generate candidate equations and apply their respective search methods to find the valid equation candidates.
$\\$
For various valid and invalid candidates related to the van der Waals equation, the existence of the Moore-Penrose inverse and validity of the subsequent equation $(AA^+-I)\vec{b}=0$ for each of their corresponding data matrices was verified in Julia.
$\\$
In determining the existence of the Moore-Penrose inverse, Gaussian elimination is applied to the matrix $A^*A$. In turning the matrix into row echelon form, if the algorithm detects a diagonal entry that is within some bound $\epsilon=0.0001$ of zero, the matrix is determined to be not invertible. If all diagonal entries are outside that bound, the matrix is considered left invertible.
$\\$
The left Moore-Penrose pseudoinverse is calculated as $(A^*A)^{-1}A^*$. $(A^*A)^{-1}$ is calculated by performing the same Gauss-Jordan operations in reducing matrix $A^*A$ to reduced row echelon form to the identity matrix.
$\\$
In determining the validity of some equation candidate, we compute, based on the values of the equation matrix $A$, pseudoinverse $A^+$, and $\vec{b}$, the vector
\begin{equation}
\label{endmat}
\vec{b'} = AA^+ \vec{b}.
\end{equation}
Construct some vector $\vec{c}$ using (\ref{endmat}), such that \begin{equation}
\label{conc}
\vec{c}[i] = \frac{\vec{b}[i] - \vec{b'}[i]}{\vec{b}[i]}, \quad 1 \leq i \leq \dim(\vec{b}).
\end{equation}
After the construction of $\vec{c}$ from (\ref{conc}), if the result of $\rVert \vec{c} \rVert$ is within some bound $\epsilon = 0.0001$ of zero, the associated candidate equation is judged as valid. If the equation is outside the bound, the equation candidate is judged as not valid.
$\\$
Below is a selection of outputs by the exhaustive search algorithm.
$\\$
$\\$
$\vdots$
$\\$
$PV + PVT + NT \implies k_2= 2.3142 \quad k_3 = -2.1348$
$\\$
$PV + PVT + NT^2 \implies k_2 = 1.7452 \quad k_3 = -1.3145$
$\\$
$\vdots$
$\\$
$PV^3+bnPV^2+n^2V+n^3+nV^2T \implies k_2= -2.661\times 10^-5 \quad k_3 = 2.45 \times 10^-2 $
$\\$
$k_4 = 6.51935 \times 10^-7 \quad k_5 = -8.3145$
$\\$
$\\$
In applying the exhaustive search and finite field sieve, both algorithms yielded the same valid equation $PV^3+nPV^2+n^2V+n^3+nV^2T$ when $\epsilon = 0.0001$. Both algorithms also rejected all other equation candidates in our search family of equations.

\section{Concluding Remarks}
\label{conclusion}
We have demonstrated a novel approach in deriving compact laws from a data set. In this approach, an algebraic model is derived to verify whether a certain algebraic equation fits the data set. With this model, we solve the problem of determining constants for equations with additive terms. We have also developed algorithms to parse through a family of algebraic equations, one by exhaustive search and the other based on the finite field sieve, which is more efficient. In addition, we prove that both algorithms are guaranteed to converge in finding an equation that describes a data set if the equation belongs to the family in which the algorithms are applied.
$\\$
$\\$
The devised algebraic equation verification algorithm runs in O($tr^2$) time. $r$ is the number of experiments done, and $t = \max(dn,r)$, with $d$ being the number of additive terms of the candidate, and $m,n$ being the number of input and output variables in the data. We have proved that the question of equation validation with respect to constants is exactly the question of the existence of the pseudoinverse $A^+$ and the calculation of $(AA^+ - I)\vec{b}$ (Theorem \ref{leftinverse}). A numerical method was devised to determine whether a pseudoinverse exists with respect to a bound. Another method was used to determine whether $(AA^+ - I)\vec{b}$ is within some bound of $\vec{0}$ using squared residuals. This algorithm succeeded in validating the van der Waals equation and calculating its constants, while rejecting all other equation candidates. The exhaustive search algorithm runs in $e^{O(1)rs}$ time, where $r$ is the number of experiments in the data and $s$ is the maximum number of additive terms in the family of equations to be searched. The finite field sieve algorithm improves on our exhaustive search algorithm by converting our data into elements of a finite field to reduce the number of equations needed to be searched. This algorithm runs in $e^{O(1)\sqrt{r}s}$ time, which makes the algorithm run in sub-exponential time with respect to the number of experiments. Both algorithms searched through the same family of algebraic equations relating to the van der Waals equation and converged on the valid equation.
$\\$
$\\$
The proposed approach transforms the problem of deriving meaning from data into formulating a linear algebra model and finding equations that fit the data. Such a formulation allows the finding of equation candidates in an explicit mathematical manner. For a certain type of compact theory, our approach assures convergence, with the discovery being computationally efficient and mathematically precise. However, several limitations exist. One is that not all natural laws are in algebraic form.  For example, for an RLC circuit \cite{rlc}, which includes exponents, sinusoids, and complex numbers, applying our algorithms to data of this type would yield a Taylor approximation of the equation, which may not be accurate for larger values. Another lies in the fact that the Finite Field Sieve algorithm is still relatively slow due to it being exponential with respect to the maximum number of additive terms to be searched through. These problems may limit the algorithm's effectiveness on certain data sets. To improve this approach, further work includes finding a mathematical analogue of this process applicable to vectors. Currently, all vector algebraic equations must be found component-wise. In addition, applying certain transforms (e.g.,~Fourier or Laplace) on exponential, logarithmic, or sinusoidal equation candidates may expand the number of data types that can fit an equation.
$\\$
$\\$
We believe the most promising direction is determining a search algorithm based on linear algebra such that valid candidate equations can be discovered with high probability. This is due to the candidate equation evolution dependent on multiplication by diagonal matrices, which do not change a matrice's eigenvector space \cite{linalg}. There are several potential search criterion that are worthy of study to use as a tool for supervised training. This may lead to a probabilistic algorithm that runs in polynomial time. This re-expression of the problem of derivation of natural laws from data into a linear algebra model creates enormous potential for refining constant evaluation,
search, and learning algorithms.

\section{Materials}
The Julia code is available with this pre-print on arXiv. The code is distributed under a Creative Commons Attribution 4.0 International Public License. If you use this work please attribute to Wenqing Xu and
Mark Stalzer, Deriving compact laws based on algebraic formulation of a data set, arXiv, 2017.

\section{Acknowledgements}
This research was funded by the Gordon and Betty Moore Foundation through Grant GBMF4915 to the Caltech Center for Data-Driven Discovery.
$\\$
$\\$
This researched was conducted as the named Caltech SURF program of Dr. Jane Chen.

\bibliographystyle{ieeetr}
\bibliography{6_15_17_Discovery}   
\end{document}